\documentclass{article}

\usepackage{ijcai17}

\usepackage{graphicx}
\usepackage{array}
\usepackage{multirow}
\usepackage{multicol} 
\usepackage{mathrsfs}
\usepackage{algorithmic}
\usepackage[inoutnumbered,linesnumbered,ruled,vlined]{algorithm2e}
\usepackage{epsfig}
\usepackage{epstopdf}
\usepackage{subfigure}
\usepackage{enumerate}
\usepackage{times}
\usepackage{helvet}
\usepackage{courier}
\usepackage{url}
\usepackage{amsmath,amssymb,amsthm}
\usepackage{color}
\usepackage{caption}
\usepackage{xfrac}
\usepackage{bbding}
\usepackage{pifont}
\usepackage{rotating}
\usepackage[table]{xcolor}
\usepackage{mathtools}
\newtheorem{defn}{Definition}
\newtheorem{rmk}{Remark}

\newtheorem{ex}{Example}
 
\newtheorem{lemma}{Lemma}

\newcommand{\gnew}[2]{\mathcal{G}^{#1,#2}_{new}}
\newcommand{\ginv}[2]{\mathcal{G}^{#1,#2}_{inv}}
\newcommand{\gobs}[2]{\mathcal{G}^{#1,#2}_{obs}}

\def\Plus{\texttt{+}}
\def\Minus{\texttt{-}}

\DeclarePairedDelimiter\norm{\lVert}{\rVert}%

\newcommand{\skipNow}[1]{}
\newcommand{\todo}[1]{{\color{red}#1}}

\DeclareMathAlphabet{\mathpzc}{OT1}{pzc}{m}{it}
\DeclareMathAlphabet\mathbfcal{OMS}{cmsy}{b}{n}

\begin{document}

\title{Learning from Ontology Streams with Semantic Concept Drift\\\footnote*{\small{Preprint of paper accepted at IJCAI 2017 - to be presented at Melbourne, Australia in August 2017}}}


\author{
Freddy L\'ecu\'e\\
INRIA, France\\
Accenture Labs, Ireland
\And
Jiaoyan Chen\\
Zhejiang University\\
China
\And
Jeff Z. Pan\\
University of Aberdeen\\
United Kingdom
\And
Huajun Chen\\
Zhejiang University\\
China
}

\maketitle

\begin{abstract}
%
Data stream learning has been largely studied for extracting knowledge structures from continuous and rapid data records. 
In the semantic Web, data is interpreted in ontologies and its ordered sequence is represented as an ontology stream. 
Our work exploits the semantics of such streams to tackle the problem of concept drift i.e., unexpected changes in data distribution, causing most of models to be less accurate as time passes. 
To this end we revisited (i) semantic inference in the context of supervised stream learning, and (ii) models with semantic embeddings. 
The experiments show 
accurate prediction with 
data from Dublin and Beijing.
\end{abstract}

\section{Introduction and Related Work}

Stream learning, or the problem of extracting and predicting knowledge from temporal evolution of data, has been largely studied.
%
%
%
Most of techniques in Database e.g.,
\cite{CheHNW96}, adapting {\tt Apriori} \cite{AgrMSTV96} for 
streams, 
focus on 
syntactic representation 
of data to 
identify frequent associations and exploit them for prediction.
%
%
\cite{LeeCL03}
improved its scalability by partitioning all streams using a sliding-window filtering.
%
%
Approaches in Machine Learning e.g., \cite{GamK11} focus on learning decision rules 
for classifying data from streams in real-time. 
Although highly scalable, most approaches have been shown to be non robust to \emph{concept drift} i.e., unexpected changes in data distribution 
\cite{coble2000real}.
Indeed their models, built on old data and then inconsistent with new data, are less accurate as time passes.
Towards this challenge 
%
\cite{ChuZLTT11} applied online active learning using customized properties of weighing,
Alternatively \cite{gao2007general} prioritized recent data during the elaboration of the learning model through regular updates, assuming temporally adjacent data is the most representative information for prediction.
\cite{cao2003support} trained an adaptive support vector machine by placing higher weight on the errors from recent training samples.
%
\cite{Kolter2007} identify multiple candidate models learnt from different historical samples and adopt a dynamic weighted majority strategy.
\cite{bifet2015efficient} go further by considering dynamic sliding windows. 
%
Although such approaches manage gradual changes, they fail in maintaining high accuracy for sudden,  
abrupt changes. This is mainly 
due to inconsistent evolution of knowledge and lack of metrics to understand the semantics of its changes and concept drifts.

%
Towards this issue we consider their representation in the semantic Web. 
%
%
Such streams, represented as ontology streams \cite{HuaS05}, 
are 
evolutive versions of ontologies where OWL (Web Ontology Language), which is underpinned by Description Logics (DL) \cite{BaaN03}, is used as a rich description language.
%
%
%
%
From 
knowledge materialization \cite{DBLP:conf/ijcai/BeckDE16,galarraga2013amie}, 
to 
predictive reasoning \cite{DBLP:conf/ijcai/Lecue15}, 
%
all are 
inferences 
where 
dynamics, 
semantics of data are exploited 
for deriving a priori knowledge from pre-established (certain) statements.
However concept drift is not handled, 
which limits 
accuracy of prediction for highly changing streams.

%
%
%
Our approach, exploiting the semantics of data streams, tackles the problem of learning and prediction with concept drifts. 
Given some continuous knowledge, how to manage its changes and their inconsistent evolution to ensure accurate prediction?
%
%
Semantic reasoning and machine learning have been combined by revisiting features embeddings as semantic embeddings i.e., vectors capturing consistency and knowledge entailment in ontology streams. 
Such embeddings are then exploited in a context of supervised stream learning to learn models, which are robust to concept drifts i.e., sudden and inconsistent prediction changes. 
Our approach has been shown to be adaptable and flexible to basic learning techniques.
The experiments have shown accurate prediction with 
live stream data from Dublin in Ireland and Beijing in China.

%
%
%
Next section reviews the adopted logic and ontology stream learning problem. 
In Section 3 we study concept drift 
and its significance.
Section 4 presents how semantic embeddings are elaborated and exploited to derive accurate prediction.
%
Finally, we report experimental results on accuracy with data from Dublin and Beijing and draw some conclusions

\section{Background}{\label{sec:Background}}

The semantics of 
data 
is represented using an ontology. 
We focus on Description Logic (DL) 
to define ontologies since it
offers 
reasoning support for most of its expressive families and compatibility to W3C standards e.g., OWL 2. 
%
%
Our work is illustrated using DL $\mathcal{EL}^{++}$ \cite{BaaBL05}, which supports  polynomial time reasoning.
%
We review 
(i) DL basics of $\mathcal{EL}^{++}$, 
(ii) ontology stream,
(iii) stream learning problem.

\subsection{Description Logics $\mathcal{EL}^{++}$ }

%
%

A signature $\Sigma$, noted 
$(\mathcal{N}_C, \mathcal{N}_R, \mathcal{N}_I)$ consists of $3$ disjoint sets of (i) atomic concepts $\mathcal{N}_C$, (ii) atomic roles $\mathcal{N}_R$, and (iii) individuals $\mathcal{N}_I$. 
Given a signature, the top concept $\top$, the bottom concept $\bot$, an atomic concept $A$, an individual $a$, an atomic role expression $r$, $\mathcal{EL}^{++}$ concept expressions $C$ and $D$ in $\mathcal{C}$ can be composed with the following constructs:
\begin{equation}
\top\;|\;\bot\;|\;A\;|\;C\sqcap D\;|\;\exists r.C\;|\;\{a\}\nonumber
\end{equation}
%
%
The 
DL
ontology $\mathcal{O}\stackrel{.}{=}\langle\mathcal{T}, \mathcal{A}\rangle$ is composed of 
TBox $\mathcal{T}$, 
ABox $\mathcal{A}$. 
A TBox is a set of concept,
role axioms. $\mathcal{EL}^{++}$ supports General Concept Inclusion axioms (GCIs e.g. $C \sqsubseteq D$), 
Role Inclusion axioms (RIs e.g., $r \sqsubseteq s$
).
An ABox is a set of concept assertion axioms e.g., $C(a)$, role assertion axioms e.g., $R(a, b)$, 
individual in/equality axioms e.g., $a \neq b$,
$a = b$.
%
%

%
\begin{ex}\textbf{(TBox and ABox Concept Assertion Axioms)}\\
Figure \ref{fig:StaticOntology:Background:Knowledge} presents (i) a TBox $\mathcal{T}$ where $DisruptedRoad$ \eqref{eq:DisruptedRoad} denotes the concept of ``{roads which are adjacent to an event causing high disruption}", (ii) concept assertions (\ref{eq:adjR1}-\ref{eq:adjR2}) with the individual $r_0$ having roads $r_1$ and $r_2$ as adjunct roads.
\end{ex}

\vspace{-0.35cm}
\begin{figure}[h!]
\begin{center}
\setlength{\fboxsep}{0.2cm}
\framebox{%
\begin{minipage}[t][2.55cm]{8.05cm}
\begin{gargantuan} 
\vspace*{-0.55cm}
\begin{equation}
\hspace*{-0.5cm} 
SocialEvent\,\sqcap\,\exists type.Poetry \sqsubseteq Event\,\sqcap\,\exists disruption.Low
\label{eq:SocialEvent}\\
\end{equation}
\vspace*{-0.57cm}
\begin{equation}
\hspace*{-0.45cm}
Road\,\sqcap\,\exists adj.(\exists occur.\exists disruption.High) \sqsubseteq DisruptedRoad
\label{eq:DisruptedRoad}\\
\end{equation}
\vspace*{-0.47cm}
\begin{equation}
\hspace*{-1.20cm}
Road\,\sqcap\,\exists adj.(\exists occur.\exists disruption.Low) \sqsubseteq ClearedRoad
\label{eq:ClearedRoad}\\
\end{equation}
\vspace*{-0.47cm}
\begin{equation}
\hspace*{-2.7cm}
BusRoad\,\sqcap\,\exists travel.Long \sqsubseteq DisruptedRoad
\label{eq:DisruptedRoad2}\\ 
\end{equation}
\vspace*{-0.47cm}
%
\begin{equation}
\hspace*{-3.2cm}
BusRoad\,\sqcap\,\exists travel.OK \sqsubseteq ClearedRoad
\label{eq:CongestedBus}\\ 
\end{equation}
\vspace*{-0.47cm}
\vspace*{-0.42cm}
%
\setlength{\columnsep}{-0.45cm}
\begin{multicols}{2}\noindent
\begin{equation}
\hspace*{-0.40cm}
Road\,\sqcap\,\exists with.Bus \sqsubseteq BusRoad 
\label{eq:BusRoad2}\\ 
\end{equation}
\begin{equation}
Road(r_0)\label{eq:RoadR0}
\end{equation}
\end{multicols}
\vspace*{-0.24cm}
\vspace{-0.60cm} 
\setlength{\columnsep}{0.65cm}
\begin{multicols}{3}\noindent
\begin{equation}
\hspace*{-1.50cm}Road(r_1)
\label{eq:?}
\end{equation}
\begin{equation}
\hspace*{-1.5cm}Road(r_2)  
\label{eq:roadR2} 
\end{equation}
\begin{equation}
\hspace*{-0.7cm}Bus(b_0)
\label{eq:ConceptAssertionr2}
\end{equation}
\end{multicols}
\vspace*{-0.85cm}
%
\vspace{0.00cm} 
\setlength{\columnsep}{1.1cm}
\begin{multicols}{3}\noindent
\begin{equation}
\hspace*{-0.55cm}adj(r_0, r_1) 
\label{eq:adjR1}\\
\end{equation}
\begin{equation}
\hspace*{-2.2cm}adj(r_0, r_2)\label{eq:adjR2}
\end{equation}
\begin{equation}
\hspace*{-1.35cm}Long \sqcap OK \sqsubseteq \bot
\label{eq:?}
\end{equation}
\end{multicols}
\end{gargantuan}
\end{minipage}
}
\end{center}
\vspace{-0.35cm}
\caption{\label{fig:StaticOntology:Background:Knowledge}$\mathcal{O}\stackrel{.}{=}\langle\mathcal{T}, \mathcal{A}\rangle$. Sample of TBox $\mathcal{T}$ and ABox $\mathcal{A}$.}
\end{figure}
\vspace{-0.2cm}

%
%
%
%


All completion
rules, 
which  
are used to classify 
%
$\mathcal{EL}^{++}$ TBox $\mathcal{T}$ and entail subsumption, 
are described in \cite{BaaBL05}.
Reasoning with such 
rules is PTime-Complete.

%
\subsection{Ontology Stream}

We represent knowledge evolution by
a dynamic, 
evolutive version of ontologies 
\cite{HuaS05}.
{Data (ABox), 
its 
inferred statements (entailments) are evolving over time while its schema (TBox) remains unchanged.}

%
%
%
%

%
\begin{defn}{\label{defn:ontologyStreamDef}}\textbf{(DL $\mathcal{L}$ Ontology Stream)}\\
A DL $\mathcal{L}$ ontology stream $\mathcal{P}_m^{n}$ from point of time $m$ to point of time $n$ is a sequence of (sets of)
Abox axioms $(\mathcal{P}_m^{n}(m), \mathcal{P}_m^{n}(m\Plus1),\cdots, \mathcal{P}_m^{n}(n))$
with respect to a static 
TBox $\mathcal{T}$ in a DL $\mathcal{L}$ where $m, n\in \mathbb{N}$ and $m<n$. 
\end{defn} 

%
$\mathcal{P}_m^{n}(i)$ is 
a snapshot of an ontology stream 
$\mathcal{P}_m^{n}$ at 
time $i$, referring to 
ABox axioms.
%
Thus a transition from $\mathcal{P}_m^{n}(i)$ to $\mathcal{P}_m^{n}(i\Plus1)$ is seen as an ABox update.
We denote by $\mathcal{P}_m^{n}[i,j]$ 
{
i.e., $\bigcup_{k=i}^{j} \mathcal{P}_m^n(k)$ 
}
a 
windowed stream 
of $\mathcal{P}_m^{n}$ 
between time $i$ and $j$
with $i \leq j$.
{
Any window $[i,j]$ has a fixed length. $1$-length windows are denoted by $(i)$.}
%
%
%
We 
consider streams $\mathcal{P}_{0}^{n}$ with $[\alpha] \doteq [i,j]$,
$[\beta] \doteq [k,l]$ as 
windows 
in $[0,n]$ {and $i<k$.}
%

\begin{ex}\textbf{(DL $\mathcal{EL}^{++}$ Ontology Stream)}\\
Figure \ref{fig:DynamicOntologyStream} illustrates 
$\mathcal{EL}^{++}$ streams $\mathcal{P}_{0}^{n}$, $\mathcal{Q}_{0}^{n}$, $\mathcal{R}_{0}^{n}$, related to events, travel time, buses, through 
snapshots at 
time $i\in\{0,1,2,3\}$ (i.e., a view on 
$[0,3]$).
In our example $n$ is any integer greater than $5$.
Their dynamic knowledge is captured by evolutive ABox axioms e.g., \eqref{eq:socialEventStreamSnapshot} captures 
$e_1$ as ``a social poetry event occurring in $r_2$" at time $1$ 
of 
$\mathcal{P}_0^n$. 
\end{ex}

%
\noindent By applying 
completion rules 
on 
static knowledge $\mathcal{T}$ and 
ontology streams $\mathcal{P}_0^n$, 
snapshot-specific axioms are
inferred.

The evolution of a stream is captured along its changes 
i.e., 
\emph{new}, \emph{obsolete} and \emph{invariant} ABox entailments from one windowed stream to another one in Definition \ref{defn:ontologyStreamChangesIJCAI2015} \cite{DBLP:conf/ijcai/Lecue15}.

\begin{defn}{\label{defn:ontologyStreamChangesIJCAI2015}}\textbf{(ABox Entailment-based Stream Changes)}\\
Let $\mathcal{S}_0^n$ be a stream; 
$[\alpha]$, 
$[\beta]$ be 
windows in $[0,n]$; $\mathcal{T}$ be axioms,
$\mathcal{G}$ its ABox entailments.
%
The changes occurring from $\mathcal{S}_{0}^{n}[\alpha]$ to $\mathcal{S}_{0}^{n}[\beta]$, denoted by $\mathcal{S}_{0}^{n}[\beta] \nabla \mathcal{S}_{0}^{n}[\alpha]$, are ABox entailments 
in $\mathcal{G}$
being $new$ \eqref{eq:newSubsumedIJCAI2015}, $obsolete$ \eqref{eq:obsoleteSubsumedIJCAI2015}, 
$invariant$ \eqref{eq:invariantSubsumedIJCAI2015}. 
\begin{small}
\begin{align}
\gnew {[\alpha]} {[\beta]} &\doteq \{g\in\mathcal{G}\;|\;\mathcal{T}\cup\mathcal{S}_{0}^{n}[\beta]\models g\;\wedge \mathcal{T}\cup\mathcal{S}_{0}^{n}[\alpha]\not\models g\}\label{eq:newSubsumedIJCAI2015}\\
\gobs {[\alpha]} {[\beta]} &\doteq\{g\in\mathcal{G}\;|\;\mathcal{T}\cup\mathcal{S}_{0}^{n}[\beta]\not\models g\;\wedge \mathcal{T}\cup\mathcal{S}_{0}^{n}[\alpha]\models g\}\label{eq:obsoleteSubsumedIJCAI2015}\\
\ginv {[\alpha]} {[\beta]} &\doteq\{g\in\mathcal{G}\;|\;\mathcal{T}\cup\mathcal{S}_{0}^{n}[\beta]\models g\;\wedge \mathcal{T}\cup\mathcal{S}_{0}^{n}[\alpha]\models g\}\label{eq:invariantSubsumedIJCAI2015}
\end{align}
\end{small}
\end{defn} 

\eqref{eq:newSubsumedIJCAI2015} reflects 
knowledge we gain by sliding window from $[\alpha]$ to $[\beta]$ while 
\eqref{eq:obsoleteSubsumedIJCAI2015} and \eqref{eq:invariantSubsumedIJCAI2015} denote respectively lost and stable
knowledge. 
All duplicates are supposed removed.
Definition \ref{defn:ontologyStreamChangesIJCAI2015} provides basics, through ABox entailments, for understanding how knowledge is evolving over time.

%
\vspace{-0.1cm}
\begin{figure}[h!]
\begin{center}
\setlength{\fboxsep}{0.2cm}
\framebox{%
\begin{minipage}[t][4.15cm]{8.05cm}
\begin{gargantuan}

\vspace*{-0.55cm} 
\begin{equation} 
\hspace*{-0.44cm}\mathcal{P}_{0}^{n}(0): (Incident\,\sqcap\,\exists impact.Limited)(e_3),\;\;occur(r_1, e_3)
\label{eq:incidentSOT7}
\end{equation}
\vspace*{-0.60cm}
\begin{equation}
\hspace*{-4.05cm}
\mathcal{Q}_{0}^{n}(0): (Road\,\sqcap\,\exists travel.OK)(r_1)
\label{eq:P097}\end{equation}
\vspace*{-0.47cm} 
\begin{equation}
\hspace*{-5.9cm}
\mathcal{R}_{0}^{n}(0): with(r_1, b_0)
\label{eq:Q097}
\end{equation}

\vspace*{-0.63cm}
\begin{equation}
\hspace*{-0.48cm}
\mathcal{P}_{0}^{n}(1): (SocialEvent\,\sqcap\,\exists type.Poetry)(e_1),\;\;occur(r_2, e_1)
\label{eq:socialEventStreamSnapshot}
\end{equation} 
\vspace*{-0.60cm}
\begin{equation}
\hspace*{-0.28cm}
\hspace*{-3.77cm}
\mathcal{Q}_{0}^{n}(1): (Road\,\sqcap\,\exists travel.OK)(r_2)
\label{eq:Q096}
\end{equation}
\vspace*{-0.47cm} 
\begin{equation}
\hspace*{-5.90cm}
\mathcal{R}_{0}^{n}(1): with(r_2,b_0)
\label{eq:?}
\end{equation}

\vspace*{-0.65cm} 
\begin{equation} 
\hspace*{-0.63cm}\mathcal{P}_{0}^{n}(2): (Event\,\sqcap\,\exists disruption.High)(e_2),\;\;occur(r_2, e_2)
\label{eq:incidentSOT7}
\end{equation}
\vspace*{-0.60cm}
\begin{equation}
\hspace*{-3.85cm} 
\mathcal{Q}_{0}^{n}(2): (Road\,\sqcap\,\exists travel.Long)(r_2)
\label{eq:Q095}\end{equation}
\vspace*{-0.47cm} 
\begin{equation}
\hspace*{-5.90cm}
\mathcal{R}_{0}^{n}(2): with(r_2,b_0)
\label{eq:Q097}
\end{equation}

\vspace*{-0.65cm} 
\begin{equation} 
\hspace*{-0.63cm}\mathcal{P}_{0}^{n}(3): (Event\,\sqcap\,\exists disruption.High)(e_2),\;\;occur(r_2, e_2)
\label{eq:incidentSOT7}
\end{equation}
\vspace*{-0.60cm}
\begin{equation}
\hspace*{-3.85cm} 
\mathcal{Q}_{0}^{n}(3): (Road\,\sqcap\,\exists travel.Long)(r_2)
\label{eq:Q095B}\end{equation}
\vspace*{-0.47cm} 
\begin{equation}
\hspace*{-5.90cm}
\mathcal{R}_{0}^{n}(3): with(r_2,b_0)
\label{eq:Q097B}
\end{equation}
\end{gargantuan}
\end{minipage}
} 
\end{center}
\vspace{-0.35cm}
\caption{\label{fig:DynamicOntologyStream}\hspace{-0.003cm}Ontology Streams $\mathcal{P}_{0}^{n}(i), \mathcal{Q}_{0}^{n}(i), \mathcal{R}_{0}^{n}(i)_{i\in\{0,1,2,3\}}$.}
\end{figure}
\vspace{-0.4cm}

%
\begin{ex}{\label{ex:ontologyStreamChangesIJCAI2013}}\textbf{(ABox Entailment-based Stream Changes)}\\
Table \ref{tab:ontologyStreamChangesIJCAI2015} illustrates changes occurring from $(\mathcal{Q}\cup\mathcal{R})_{0}^{n}[0,1]$ to $(\mathcal{Q}\cup\mathcal{R})_{0}^{n}[2,3]$ through ABox entailements.
For instance ``$r_2$ as a disrupted road window $[2,3]$ of $(\mathcal{Q}\cup\mathcal{R})_{0}^{n}$ is $new$ with respect to knowledge in $[0,1]$.
It is entailed using DL completion rules on \eqref{eq:DisruptedRoad2}, \eqref{eq:BusRoad2}, \eqref{eq:roadR2}, \eqref{eq:Q095}, \eqref{eq:Q097}, \eqref{eq:Q095B} and \eqref{eq:Q097B}.
\end{ex}

\vspace{-0.25cm}
\begin{table}[h!]
\begin{smallermathTable}
\centering
\begin{tabular}[t]{@{ }c@{ }|c|c|c}
\hline

Windowed Stream & \multicolumn{3}{|c}{$(\mathcal{Q}\cup\mathcal{R})_{0}^{n}[2,3]\,\nabla\,(\mathcal{Q}\cup\mathcal{R})_{0}^{n}[0,1]$} \\\cline{2-4}
 
Changes& $obsolete$ & $invariant $ & $new$\\\hline
 
$with(r_2,b_0)$ & & \checkmark & \\\hline

$ClearedRoad(r_2)$  & \checkmark & &\\\hline

$DisruptedRoad(r_2)$ & & & \checkmark \\\hline

\end{tabular}
\vspace{-0.15cm}
\caption{\label{tab:ontologyStreamChangesIJCAI2015}ABox Entailment-based Stream Changes.}
\end{smallermathTable}
\end{table}
\vspace{-0.35cm}

\subsection{Ontology Stream Learning Problem}


%
Definition \ref{def:OSL} revisits classic supervised learning \cite{domingos2000mining} for ontology stream as the problem of predicting knowledge (through entailment) in a future snapshot.



%
\begin{defn}{\label{def:OSL}}\textbf{(Ontology Stream Learning Problem)}\\
Let $\mathcal{S}_0^{n}$ be a stream; $\mathcal{T}$, $\mathcal{A}$ be respectively TBox, ABox; $g\in\mathcal{G}$ an ABox entailment.
An Ontology Stream Learning Problem, noted OSLP$\langle \mathcal{S}_0^n, k, \mathcal{T}, \mathcal{A}, g\rangle$, is 
the problem of estimating whether $g$ can be entailed from $\mathcal{T}$ and 
$\mathcal{A}$ at 
time $k\in(0,n]$ of stream $\mathcal{S}_0^n$, given knowledge at time $t < k$ of $\mathcal{S}_0^n$.
\end{defn}

This estimation is denoted as 
$p_{|\mathcal{T}\cup\mathcal{A}}(\mathcal{S}_{0}^{n}(k)\models g)$ with values in $[0,1]$ {and $k\geq 1$. $g$ is a class assertion entailment in the form of  $G(a)$, with $G$ a concept expression and $a$ an individual.}
The {estimation}, adapted from \cite{DBLP:conf/icdm/GaoFH07}, can be elaborated using knowledge from {previous snapshots of $\mathcal{S}_0^{k}$:}
%
\begin{small}
\begin{equation}{\label{eq:conditionalProbability}}
p_{|\mathcal{T}\cup\mathcal{A}}(\mathcal{S}_{0}^{n}(k)\models g) \doteq \frac{p_{|\mathcal{T}\cup\mathcal{A}}(\mathcal{S}_0^{k\Minus 1}\models g)}{p_{|\mathcal{T}\cup\mathcal{A}}(a \in \mathcal{S}_0^{k\Minus 1})}
\end{equation}
\end{small}

\noindent {Estimation $p_{|\mathcal{T}\cup\mathcal{A}}(\mathcal{S}_{0}^{k \Minus 1}\models g)$ is 
the proportion of snapshots 
in $\mathcal{S}_0^{k \Minus 1}$ 
entailing $g$. 
The conditional probability of  $a$  in $\mathcal{S}_0^{k\Minus 1}$ (noted $a \in \mathcal{S}_0^{k\Minus 1}$)  given  $\mathcal{S}_0^{k\Minus 1}$ entailing $g$, or $G(a)$, is 1.}
\begin{ex}{\label{ex:OSL}}\textbf{(Ontology Stream Learning Problem)}\\
The problem of estimating whether class assertion $g$, defined as {$DisruptedRoad(r_2)$}, can be entailed from $\mathcal{T}$ and $\mathcal{A}$ at time $4$ of $(\mathcal{Q}\cup\mathcal{R})_0^n$ is defined as OSLP$\langle (\mathcal{Q}\cup\mathcal{R})_0^n, 4, \mathcal{T}, \mathcal{A}, g\rangle$.
The estimation can be retrieved 
using {\eqref{eq:conditionalProbability} hence $p_{|\mathcal{T}\cup\mathcal{A}}((\mathcal{Q}\cup\mathcal{R})_0^n(4)\models DisruptedRoad(r_2)) \doteq \sfrac{2}{3}$.}
%
%
%
\end{ex}

%


%
\section{Concept Drift in An Ontology Stream}{\label{sec:Subsection:TODO}}

We introduce semantic concept drift, as a basis for qualifying, 
quantifying sudden and abrupt changes in an ontology stream.

\subsection{Semantic Concept Drift}

Definition \ref{def:SCD} revisits 
\emph{concept drift} \cite{gao2007general} for ontology streams as \emph{prediction changes} (Definition \ref{def:SC}) in ABox entailment, 
which are 
\emph{sudden} 
and 
\emph{abrupt} (Definitions \ref{def:SSC}, \ref{def:ASC}).

%
\begin{defn}{\label{def:SC}}\textbf{(Prediction Change)}\\
Let $\mathcal{S}_0^n$ be a stream; $\mathcal{T}$, $\mathcal{A}$ and $\mathcal{G}$ be 
TBox, Abox and its 
entailments.
A 
prediction change in $\mathcal{S}_0^n$ is ocuring between 
time $i$ and $j$ in $[0,n]$ 
with respect to $\mathcal{T}$, $\mathcal{A}$ and its entailments iff:
%
\begin{small}
\begin{equation}{\label{eq:predictionChange}} 
\exists g\in\mathcal{G} : \norm{p_{|\mathcal{T}\cup\mathcal{A}}(\mathcal{S}_0^n(i)\models g) - p_{|\mathcal{T}\cup\mathcal{A}}(\mathcal{S}_0^n(j)\models g)} \geq \varepsilon
\end{equation}
\end{small}
\noindent where $\varepsilon \in (0,1]$ is a variable bounding the difference of estimation, $\norm{v}$ refers to the absolute value of $v$, and $j > i$
\end{defn}

{ABox entailment $g$ is called an evidence entailment of the prediction change. We denote by $\mathbb{C}_{|\mathcal{T}\cup\mathcal{A}}(\mathcal{S}_0^n, i, j, \varepsilon)$, the set of all evidence entailments of the prediction change with an $\varepsilon$ difference between time $i$ and $j$ of ontology stream $\mathcal{S}_0^n$.}


%
\begin{ex}{\label{ex:SC}}\textbf{(Prediction Change)}\\
$g \doteq DisruptedRoad(r_2)$ can be entailed from $\mathcal{T}$ and $\mathcal{A}$ at time $2$ of $(\mathcal{Q}\cup\mathcal{R})_0^n$ with a zero probability following \eqref{eq:conditionalProbability}. Therefore a prediction change between times $2$ and $4$ (cf. Example \ref{ex:OSL}) is captured with $g \in \mathbb{C}_{|\mathcal{T}\cup\mathcal{A}}((\mathcal{Q}\cup\mathcal{R})_0^n, 2, 4,\sfrac{1}{3})$.
\end{ex}

\begin{defn}{\label{def:SSC}}\textbf{($\alpha$-Sudden Prediction Change)}\\
A prediction change at point of time $i$ in stream $\mathcal{S}_0^n$, satisfying \eqref{eq:predictionChange}, is defined as 
$\alpha$-sudden, with $\alpha\in(0,n\Minus i]$ iff $j = i+\alpha$.
\end{defn}

\begin{defn}{\label{def:ASC}}\textbf{(Abrupt Prediction Change)}\\
A prediction change, satisfying \eqref{eq:predictionChange}, is abrupt iff $\exists g'\in\mathcal{G}$ s.t.
%
\begin{small}
\begin{equation}{\label{eq:APC}}
\mathcal{T} \cup \mathcal{A} \cup g \cup g' \bigcup_{k=0}^{\max\{i,j\}}\mathcal{S}_0^n(k) \models \bot
\end{equation}
\end{small}
\noindent where $\bigcup_{k=0}^{\max\{i,j\}}\mathcal{S}_0^n(k)$ captures all axioms from any snapshot $\mathcal{S}_0^n(k)$ of stream $\mathcal{S}_0^n$ with $k\in [0, \max\{i,j\}]$.
\end{defn}

Suddenness characterises the proximity of prediction changes in streams i.e., the lower $\alpha$ the closer the changes.
Abruptness captures disruptive changes from a semantic perspective i.e., conflicting knowledge among 
snapshots $\mathcal{S}_0^n(i)$,
$\mathcal{S}_0^n(j)$ with respect to background knowledge $\mathcal{T} \cup \mathcal{A}$.

\begin{defn}{\label{def:SCD}}\textbf{(Semantic Concept Drift)}\\
A semantic concept drift in $\mathcal{S}_0^n$, is defined as a $1$-sudden and abrupt prediction change.
\end{defn}

Evaluating if a 
concept drift occurs for a snapshot update is in worst case polynomial time with respect to acyclic TBoxes and $\mathcal{S}_0^n$ in $\mathcal{EL}^{++}$ since subsumption and 
satisfiability in \eqref{eq:predictionChange},
\eqref{eq:APC} can be checked in polynomial time \cite{BaaBL05}.

%

%
\begin{ex}{\label{ex:SCD}}\textbf{(Semantic Concept Drift)}\\
Two prediction changes from time $i=2$ to $3$ and $3$ to $4$ (cf. Table \ref{tab:ontologyStreamChangesIJCAI2017}) have occurred for $g \doteq DisruptedRoad(r_2)$ in $(\mathcal{Q}\cup\mathcal{R})_0^n$. They are  
semantic concept drifts as they are $1$-sudden and abrupt with $g' \doteq ClearedRoad(r_2)$ in $(\mathcal{Q}\cup\mathcal{R})_0^n(1)$.
\end{ex}

\vspace{-0.25cm}
\begin{table}[h!]
\begin{smallermathTable}
\centering
\begin{tabular}[t]{@{ }c@{ }|@{ }c@{ }|@{ }c@{ }|@{ }c@{ }|@{ }c@{ }}\hline
\multicolumn{3}{c}{Prediction} & \multicolumn{2}{c}{{Prediction Change}}\\
Past Points & Time  & $p_{|\mathcal{T}\cup\mathcal{A}}$& $g\in\mathbb{C}_{|\mathcal{T}\cup\mathcal{A}}$ & Abrupt- \\
of Time & $i$ & $((\mathcal{Q}\cup\mathcal{R})_0^n(i)\models g)$ &$((\mathcal{Q}\cup\mathcal{R})_0^n, i, i \Plus 1, \sfrac{1}{3})$ & ness \\\hline

$\{0\}$ & $1$ & $0$ & \ding{55} & \ding{55}\\\cline{1-5}

$\{0,1\}$ & $2$ & $0$ & \checkmark & \checkmark\\\cline{1-5}

$\{0,1,2\}$ & $3$ & $\sfrac{1}{2}$ & \checkmark & \checkmark\\\cline{1-5}

$\{0,1,2,3\}$ & $4$ & $\sfrac{2}{3}$ & N/A & N/A\\\hline

\end{tabular}
\vspace{-0.15cm}
\caption{\label{tab:ontologyStreamChangesIJCAI2017}\hspace{-0.056cm}Prediction Changes in $(\mathcal{Q}\cup\mathcal{R})_0^n$ $($\small{$g \doteq Disrupted(r_2)$}$)$.}
\end{smallermathTable}
\end{table}
\vspace{-0.35cm}

%
\subsection{Significance of Concept Drift}

Significance of semantic concept drift (Definition \ref{def:Significance}) is an indicator on its severity.
%
It captures the homogeneity of the concept drift across ABox entailments 
%
as the proportion of ABox entailments from $\mathcal{S}_0^n(i)$ and $\mathcal{S}_0^n(i\Plus 1)$ causing semantic concept drift. The values of significance range in $[0,1]$. 

\vspace{0.1cm}
\begin{defn}{\label{def:Significance}}\textbf{(Semantic Concept Drift Significance)}\\
The significance of a semantic 
concept drift, defined between points of time $i\in(0,n)$ and $i\Plus 1$ of 
$\mathcal{S}_0^n$ with $\varepsilon$, 
$\mathcal{T}$, $\mathcal{A}$, $\mathcal{G}$ as difference, TBox, ABox, 
and entailments, is:
%
\begin{small}
\begin{equation}{\label{eq:Significance}}
\hspace{-0.2cm}\sigma_{|\mathcal{T}\cup\mathcal{A}}(\mathcal{S}_0^n, i, \varepsilon)\doteq \frac{|\mathbb{C}_{|\mathcal{T}\cup\mathcal{A}}(\mathcal{S}_0^n, i, i \Plus 1, \varepsilon)|}{|\{g\in\mathcal{G}\;|\;\mathcal{S}_0^n(i)\models g \vee \mathcal{S}_0^n(i \Plus 1)\models g\;\}|}
\end{equation}
\end{small}
\noindent where the expression in between $|$ refers to its cardinality.
\end{defn}

Evaluating \eqref{eq:Significance} is in worst case polynomial time cf. complexity of Definition \ref{def:SCD}.

\begin{ex}{\label{ex:Significance}}\textbf{(Semantic Concept Drift Significance)}\\
By applying \eqref{eq:Significance} on 
concept drifts of Table \ref{tab:ontologyStreamChangesIJCAI2017} we derive that $\sigma_{|\mathcal{T}\cup\mathcal{A}}((\mathcal{Q}\cup\mathcal{R})_0^n, 2, \sfrac{1}{3})$ is $\sfrac{4}{7}$ while $\sigma_{|\mathcal{T}\cup\mathcal{A}}((\mathcal{Q}\cup\mathcal{R})_0^n, 3, \sfrac{1}{3})$ is $0$, hence a more significant 
drift between times $2$,
$3$ than $3$, 
$4$.
%
In other words conflicting facts $g \doteq DisruptedRoad(r_2)$ and $g'\doteq ClearedRoad(r_2)$ w.r.t. $\mathcal{T}$ and $\mathcal{A}$ have the most significant impact on prediction changes at times $2$ and $3$.
\end{ex}

\begin{lemma}\textbf{(Semantic Concept Drift Evolution)}\label{lemma:drift}\\
A semantic concept drift in any ontology stream $\mathcal{S}_0^n$ is more significant at 
time $i > 0$ than at time $i\Plus1$ if $|\gnew {[0,i]} {[0,i\Plus1]}| = 0$.
\end{lemma}

\begin{proof} 
(Sketch) 
Since $|\gnew {[0,i]} {[0,i\Plus1]}| = 0$, 
$\mathcal{S}_0^n(i)$ and $\mathcal{S}_0^n(i\Plus1)$ are similar w.r.t $\models_{\mathcal{T}\cup\mathcal{A}}$. Thus, 
the set of all entailments, predicted at 
$i\Plus 1$ and $i\Plus 2$ from \eqref{eq:conditionalProbability}, are similar 
but with different prediction values \eqref{eq:predictionChange} $\forall \varepsilon \geq 0$.
%
So $\sigma_{|\mathcal{T}\cup\mathcal{A}}(\mathcal{S}_0^n, i, \varepsilon)$ and $\sigma_{|\mathcal{T}\cup\mathcal{A}}(\mathcal{S}_0^n, i\Plus1, \varepsilon)$ in \eqref{eq:Significance} have same denominators while $\mathbb{C}_{|\mathcal{T}\cup\mathcal{A}}(\mathcal{S}_0^n, i\Plus1, i \Plus2, \varepsilon) \subseteq \mathbb{C}_{|\mathcal{T}\cup\mathcal{A}}(\mathcal{S}_0^n, i, i \Plus 1, \varepsilon)$ hence 
$\sigma_{|\mathcal{T}\cup\mathcal{A}}(\mathcal{S}_0^n, i\Plus1, \varepsilon) \leq \sigma_{|\mathcal{T}\cup\mathcal{A}}(\mathcal{S}_0^n, i, \varepsilon)$.
\end{proof}

Algorithm \ref{algo:rulesInterestingnessUpdate} {\tt [A1]} retrieves significant 
concept drifts in $\mathcal{S}_0^n$ with minimal significance $\sigma_{\min}$. 
{\tt [A1]} iterates on all snapshots updates except those with no new ABox entailment (line \ref{algo:Drift:Lemma} - lemma \ref{lemma:drift}) for minimizing satisfiability and subsumption checking. 
Semantic concept drifts, as $1$-sudden and abrupt prediction changes, are retrieved (line \ref{algo:Drift:SCD}).
{\tt [A1]} completes the process (line \ref{algo:Drift:SSCD}) by filtering drifts by significance $\sigma_{\min}$.
%

%
\vspace{-0.4cm}
\begin{algorithm}[h!]
\small

\KwIn{
(i) Axioms $\mathcal{O}:\langle\mathcal{T}, \mathcal{A}, \mathcal{G}\rangle$,
(ii) Ontology stream $\mathcal{S}_0^n$,
(iii) Lower limit $\varepsilon\in(0,1]$ of prediction difference,
(iv) Minimum threshold of drift significance $\sigma_{\min}$.
} 

\KwResult{
$\mathbb{S}$: $\mathbb{S}$ignificant concept drifts in $\mathcal{S}_0^n$ w.r.t. $\sigma_{\min}$.
}

\Begin{
$\mathbb{S}\leftarrow \emptyset$;
\emph{\% Init. of the $\mathbb{S}$ignificant concept drifts set. 
}\\
%
\ForEach{$i\in (0,n]$ of $\mathcal{S}_0^n$ such that $|\gnew {[0,i]} {[0,i\Plus1]}| \neq 0$\nllabel{algo:Drift:Lemma}}{
\emph{\% Selection of $1$-sudden, abrupt prediction changes.}\\
\If{$\exists g, g'\in\mathcal{G}$ such that: \newline
    \hspace*{0.25cm}$\norm{p_{|\mathcal{T}\cup\mathcal{A}}(\mathcal{S}_0^n(i)\models g) \Minus\; p_{|\mathcal{T}\cup\mathcal{A}}(\mathcal{S}_0^n(i\Plus 1)\models g)} \geq \varepsilon$ \newline
    \hspace*{0.25cm}$\wedge\;\mathcal{T} \cup \mathcal{A} \cup \mathcal{S}_0^n(i) \cup \mathcal{S}_0^n(i\Plus 1) \cup g \cup g'\models \bot$
    \nllabel{algo:Drift:SCD}}
{
\emph{\% Semantic concept drift with min. significance.}\\
\If{$\sigma_{|\mathcal{T}\cup\mathcal{A}}(\mathcal{S}_0^n, i, \varepsilon) \geq \sigma_{\min}$\nllabel{algo:Drift:SSCD}}
{
$\mathbb{S}\leftarrow \mathbb{S} \cup \{(i, i\Plus 1)\}$ \emph{\% Add snapshot update.}
}
}
}
\Return{$\mathbb{S}$};
}
\caption{{\label{algo:rulesInterestingnessUpdate}}\small{{\tt{[A1]SignificantDrift}}$\langle \mathcal{O}, \mathcal{S}_0^n, \varepsilon, \sigma_{\min}\rangle$}}
\end{algorithm}
\vspace{-0.4cm} 

Computing a solution with {\tt [A1]} given a polynomial input $n$, number of axioms,  
entailments in $\mathcal{O}$ and $\mathcal{S}_0^n$ is in worst case polynomial time, due to the complexity of evaluating a semantic drift cf. complexity of Definition \ref{def:SCD}.
However computing significant $\alpha$-sudden,
abrupt prediction changes following {\tt [A1]} is in worst case NP w.r.t. 
the number snapshots.

\vspace{-0.1cm}
\section{Ontology Stream Learning}{\label{sec:Subsection:OSL}}

We tackle the ontology stream learning problem by (i) computing semantic embeddings, as mathematical objects exploiting the properties of concept drifts, 
(ii) applying all embeddings in model-based learning approaches (Algorithm \ref{algo:ConsistentPrediction}). 

%
\subsection{Semantic Embeddings}

The semantics of 
streams exposes two levels of knowledge which are crucial for learning with 
concept drift: (i) (in-)consistency evolution of knowledge, and (ii) entailment of the forecasting target from stream assertions and axioms. 
They are semantic embeddings, captured as: \emph{consistency vectors} (Definition \ref{def:ConsistencyVector}) and \emph{entailment vector} (Definition \ref{def:EntailmentVector}).

\begin{defn}{\label{def:ConsistencyVector}}\textbf{(Consistency Vector)}\\
A consistency vector of 
snapshot $\mathcal{S}_0^n(i)$ in
$\mathcal{S}_0^n$, denoted by ${\bf{c}}_{i}$, 
is defined $\forall j\in[0,n]$ by ${c}_{ij}$ if $i<j$; 
${c}_{ji}$ otherwise such that:
\begin{small}
\begin{equation} 
\label{eq:ConsistencyVector}
   \hspace{-0.1cm}c_{ij} \stackrel{.}{=} 
   \hspace*{-0.1cm}\left\{ \begin{array}{lcl}
         \hspace{-0.1cm} 
      \frac{|\ginv i j |}{|\gnew i j | + |\ginv i j | + |\gobs i j |} &
           \vspace*{0.2cm}\hspace{-0.18cm}\scriptsize{\text{if $\mathcal{T}\cup\mathcal{S}_0^n(i)\cup\mathcal{S}_0^n(j)\not\models \bot$}}\\ 
           \hspace{-0.1cm} 
      \frac{|\ginv i j |}{|\gnew i j | + |\ginv i j | + |\gobs i j |}  -1     & \hspace{-0.15cm}\text{otherwise}
        \end{array} \right.
\end{equation}
\end{small}
\noindent where the expressions in between $|$ refer to its cardinality i.e., the number of new \eqref{eq:newSubsumedIJCAI2015}, obsolete \eqref{eq:obsoleteSubsumedIJCAI2015}, 
invariant \eqref{eq:invariantSubsumedIJCAI2015} ABox entailments 
from $\mathcal{S}_{0}^{n}(i)$ to $\mathcal{S}_{0}^{n}(j)$. 
$c_{ij} = c_{ji}\;\forall i,j\in[0,n]$.
\end{defn}

A consistent vector, with values in $[-1,1]^{n\Plus1}$, encodes (i) (in-)consistency with (negative) positive values, and (ii) similarity of knowledge among 
$\mathcal{S}_0^n(i)$ and any other snapshot $\mathcal{S}_0^n(j)_{j\in[0,n]}$ of stream $\mathcal{S}_0^n$ w.r.t axioms $\mathcal{T}$ and $\mathcal{A}$.
%
The number of invariant entailments has a 
positive influence on \eqref{eq:ConsistencyVector}. On contrary, the number of new and obsolete ABox entailments, capturing some differentiators in knowledge evolution, has a negative impact. When an inconsistency occurs, the value $1$ is subtracted instead of considering its additive inverse. This ensures that the invariant factor has always a positive impact.

Evaluating \eqref{eq:ConsistencyVector} is in worst case polynomial time with respect to $\mathcal{T}$ and $\mathcal{S}_0^n$ in $\mathcal{EL}^{++}$. Indeed its evaluation requires (i) ABox entailment, and (ii) basic set theory operations from Definition \ref{defn:ontologyStreamChangesIJCAI2015}, both in polynomial time 
\cite{BaaBL05}.

\begin{ex}{\label{ex:ConsistencyVector}}\textbf{(Consistency Vector)}\\
%
Consistency vector ${\bf c}_3$ i.e., $(c_{03}, c_{13}, c_{23}, c_{33})$ of 
$(\mathcal{Q}\cup\mathcal{R})_0^n(3)$ is 
$(0,\Minus\;0.8,1,1)$. 
Knowledge at time $3$ is 
consistent / inconsistent / similar with knowledge at times 
$0$ / $1$ / $2$ and $3$.
\end{ex}

An entailment vector (Definition \ref{def:EntailmentVector}) is adapting the concept of feature vector \cite{bishop2006pattern} in Machine Learning to represent the (non-)presence of all ABox entailments (using $\models$ w.r.t. $\mathcal{T}$, $\mathcal{A}$) in a given snapshot.
Each dimension captures whether a particular ABox entailment is in ($1$) or not ($0$). 

\begin{defn}{\label{def:EntailmentVector}}\textbf{(Entailment Vector)}\\
Let $\mathcal{G}\doteq\{g_1,\ldots,g_m\}$ be all distinct ABox entailments in $\mathcal{S}_0^n$.
An entailment vector of a snapshot $\mathcal{S}_0^n(i)$ in $\mathcal{S}_0^n$, denoted by ${\bf e_i}$, is a vector of dimension $m$ such that $\forall j\in[0,m]$ 
\begin{small}
\begin{equation} \label{eq:EntailmentVector}
e_{ij} \stackrel{.}{=} 1\;\;{\text{if $\mathcal{T}\cup\mathcal{A}\cup\mathcal{S}_0^n(i)\models g_j$}},\;0\;\;\text{otherwise}
\end{equation}
\end{small} 
\end{defn}

\begin{rmk}{\label{rmk:fev}}\textbf{(Feature vs. Entailment Vector)}\\
Feature vectors are bounded to 
only raw data while entailment vectors, with 
much larger dimensions, embed both data and its inferred assertions from $\mathcal{T}$ and DL completion rules. The latter ensures a larger and more contextual coverage.
\end{rmk}

\vspace{-0.2cm}
\subsection{Semantic Prediction}

%
Algorithm \ref{algo:ConsistentPrediction} {\tt [A2]} 
aims at learning a 
model (line \ref{algo:ConsistentPrediction:CLM}) over ${\bf N}\leq n\Plus1$ snapshots of 
$\mathcal{S}_0^n$, noted $\mathcal{S}_0^n|_{\kappa}$, 
for prediction at 
$n\Plus1$. $\kappa$ refers to the proportion of snapshots with concept drift used for modelling.
$\mathcal{S}_0^n|_{\kappa}$ is selected to capture (i) 
$\mathcal{S}_0^n(n)$ i.e., the closest (temporally) to $\mathcal{S}_0^n(n\Plus1)$ (line \ref{algo:ConsistentPrediction:Init}), (ii) 
knowledge in the most (lines \ref{algo:ConsistentPrediction:F2}-\ref{algo:ConsistentPrediction:F3}) significant concept drifts (Definition \ref{def:Significance} - line \ref{algo:ConsistentPrediction:F1}), 
(iii) any other snapshots to meet ${\bf N}$ 
(line \ref{algo:ConsistentPrediction:F5}).

The model is trained, following Stochastic Gradient Descent method \cite{Zhang2004}, using samples of the form $\left\{({\bf e}_i, g_i)\;|\;{i\in\{1,\ldots,{\bf N}\}}\right\}$ where ${\bf e}_i$ is the entailment vector for $\mathcal{S}_0^n(i)$ and ${\bf v}(g_i)$ is the target variable in $[0,1]$, capturing the estimation of $g_i$ to be entailed. $g_i$ is determined by the entailment vector.
%
The goal is to learn a linear scoring function $f({\bf e}_i)=a^T{\bf e}_i+b$ with model parameters $a\in\mathbf{R}^{\bf N}$ and $b\in {\bf R}$ which minimizes the following objective function $O_j$:
%
\vspace{-0.19cm}
\begin{equation}\label{eq:loss}
\begin{aligned}
O_j(a,b) \doteq \sum_{i=1}^{\kappa} \omega_{ij} L({\bf v}(g_i),f({\bf e}_i)) + \alpha R(a),
\end{aligned}
\end{equation}
\vspace{-0.02cm}
\noindent where $L$ represents the loss function (e.g., Hinge for SVM or $\log$ for logistic regression).
$R$ and $\alpha$ control the variance of the model in case of over fitting. 
$R$ is a regularization term and $\alpha > 0$ is a non-negative hyperparameter.
%
%
Each sample $({\bf e}_i, g_i)$ in \eqref{eq:loss} is weighted by $\omega_{ij}$ in \eqref{eq:g1} (resp. \eqref{eq:g2})
%
for filtering out 
consistent (resp. inconsistent) historical snapshots w.r.t. \eqref{eq:ConsistencyVector}. 
$\omega_{ij}$ controls the consistency level of models.
%
%
%
\begin{small}
\setlength{\columnsep}{0.6cm}
\begin{multicols}{2}\noindent
\begin{equation}\label{eq:g1}
\hspace*{-0.2cm}\omega_{ij} \doteq 
\begin{cases} 
0,  & \mbox{if }  {c}_{ij} > 0 \\ 
\Minus {c}_{ij} & \mbox{else},
\end{cases}
\end{equation}
\begin{equation}\label{eq:g2}
\hspace*{-0.6cm}\omega_{ij} \doteq 
\begin{cases} 
0,  & \mbox{if }  {c}_{ij} < 0 \\ 
{c}_{ij} & \mbox{else},
\end{cases}
\end{equation}
\end{multicols}
\end{small}

%
\vspace{-0.35cm}
\begin{algorithm}[h!]
\small

\KwIn{
(i) Axioms $\mathcal{O}:\langle\mathcal{T}, \mathcal{A}, \mathcal{G}\rangle$,
(ii) 
Stream $\mathcal{S}_0^n$,
(iii) Lower limit $\varepsilon\in(0,1]$, 
%
(iv) Minimum 
drift significance $\sigma_{\min}$,
(v) Proportion $\kappa$ of snapshots with concept drift used for modelling, 
(vi) Number of snapshots ${\bf N}$ for modelling.
} 

\KwResult{$f$: Model for prediction at time $n\Plus1$ of $\mathcal{S}_0^n$.
}

\Begin{
$\mathcal{S}_0^n|_\kappa\leftarrow\{n\}$; \emph{\% Initial snapshots' set for learning model.}\nllabel{algo:ConsistentPrediction:Init}\\
\emph{\% Computation of the most significant drifts w.r.t. $\varepsilon$, $\sigma_{\min}$.}\nllabel{algo:ConsistentPrediction:F1}\\
$\mathbb{S}\leftarrow \tt{SignificantDrift}\langle \mathcal{O}, \mathcal{S}_0^n, \varepsilon, \sigma_{\min}\rangle$;\\
\emph{\% Selection of $\sfrac{\kappa}{{\bf N}}$ snapshots involved in concept drifts $\mathbb{S}$.}\\
\ForEach{$i\in[0,n]$ s.t. $(i, i\Plus1)\in\mathbb{S}\;\wedge\;|\mathcal{S}_0^n|_\kappa|<\sfrac{\kappa}{{\bf N}}$\nllabel{algo:ConsistentPrediction:F2}}{
$\mathcal{S}_0^n|_\kappa\leftarrow \mathcal{S}_0^n|_\kappa \cup \{i\}$;\nllabel{algo:ConsistentPrediction:F3}
}
\emph{\% Expand $|\mathcal{S}_0^n|_\kappa|$ with snapshots not involved in $\mathbb{S}$.}\\
{add $\sfrac{1\Minus\kappa}{{\bf N}}$ point(s) of time $i$ to $\mathcal{S}_0^n|_\kappa$ s.t. $(i, i\Plus1)\notin\mathbb{S}$;\nllabel{algo:ConsistentPrediction:F5}}\\
\vspace{0.1cm}
\emph{\% Learning model $f$ using \eqref{eq:loss} with weight \eqref{eq:g1} or \eqref{eq:g2} \nllabel{algo:ConsistentPrediction:CLM}.}\\
\vspace{-0.35cm}
\begin{align}
&\text{(i)}\;\;\min_{(a,b)\in{\bf R}^{\bf N}\times {\bf R}} \sum_{i=1}^{\bf N} \omega_{ij} L({\bf v}(g_i),f({\bf e}_i)) + \alpha R(a)\nonumber\\
&\text{(ii)}\;\;f({\bf e}_i)=a^T{\bf e}_i+b\nonumber
\end{align}
\Return{$f$};
}
\caption{{\label{algo:ConsistentPrediction}}\small{{\tt{[A2]PredictionModel}}$\langle \mathcal{O}, \mathcal{S}_0^n, \varepsilon, \sigma_{\min}, \kappa, {\bf N}\rangle$}}
\end{algorithm}
\vspace{-0.35cm} 

%
{\tt [A1\Minus2]} parameterized with low $\varepsilon$, 
$\sigma_{\min}$, high $\kappa$ and \eqref{eq:g1} as weight (line \ref{algo:ConsistentPrediction:CLM} (i))
favours 
models with significant concept drifts for prediction, 
which 
supports diversity and prediction changes in the model.
%
Parameterized 
with high $\varepsilon$, 
$\sigma_{\min}$, low $\kappa$ and \eqref{eq:g2} as weight, it will capture more consistent models.

%
The linear scoring function $f$ in \eqref{eq:loss} has the following advantages compared to more complex structures such as artificial neural network: (i) better handling over-fitting with reduced sample size - due to filtering of snapshots not involved in significant concept drifts (lines \ref{algo:ConsistentPrediction:F2}-\ref{algo:ConsistentPrediction:F3} in {\tt [A2]}), 
(ii) ensuring efficient, scalable learning and prediction for online contexts.

\section{Experimental Results}\label{sec:evaluation}

We report accuracy 
by 
(i) studying the impact of {\tt [A2]} and semantic embeddings on    
concept drift 
for \emph{Dublin-Ireland},
\emph{Beijing-China} 
applications, and 
(ii) comparing its results with state-of-the-art approaches. 
The system is tested on: 
16 Intel(R) Xeon(R) CPU E5-2680, 2.80GHz cores, 
32GB RAM.

\vspace{0.1cm}
\noindent $\bullet$ \textbf{Beijing Air Quality (BAQ) Context:} 
BAQ index, ranging from Good (value $5$), Moderate ($4$), Unhealthy ($3$), Very Unhealthy ($2$) to Hazardous ($1$), 
can be forecasted 
using data streams of $B_1$: air pollutants 
and meteorology elements $B_2$: wind speed, $B_3$: humidity observed in $12$ sites.
The variation of context, characterising a concept drift problem, makes BAQ index difficult to be forecasted specially with potentially erroneous sensor data. 
The semantics of context is based on a DL $\mathcal{ALC}$ ontology, including $48$ concepts, $13$ roles, $598$ axioms.
An average of $6,500$ RDF triples are generated at each update (i.e., every $600$ seconds) for all streams.

\vspace{0.1cm}
\noindent $\bullet$ \textbf{Dublin Bus Delay (DBD) Context:} 
DBD, classified as Free (value $5$), Low ($4$), Moderate ($3$), Heavy ($2$),
Stopped ($1$) can be forecasted   
using reputable live stream contextual data (Table \ref{tab:DataSets}) related to $D_1$: bus GPS location, delay, 
congestion status, $D_2$: weather conditions, 
$D_3$: road incidents.  
However bus delay is subject to major changes due the high degree of context variation. The latter, responsible for the concept drift problem, 
impacts accuracy the most.
We consider an extended settings by enriching data using a DL $\mathcal{EL}^{++}$ domain ontology ($55$ concepts, $19$ roles and $25,456$ axioms).

\vspace{-0.1cm}
\begin{table}[h!]
\vspace{-0.08cm}
\scriptsize{
\centering
\begin{tabular}[t]{l|c||c|c|c}
\hline
Feature & Size (Mb) & Frequency of & \#Axioms  & \#RDF Triples \\
DataSet&  per day & Update (seconds) & per Update  & per Update \\\hline\hline
${D}_1$: Bus & 120 & $40$& 3,000 & 12,000\\\hline
${D}_2$: Weather & 3 & $300$ & 53 & 318\\\hline
${D}_3$: Incident & 0.1 & $600$  & 81 & 324\\\hline
\end{tabular}
\vspace{-0.2cm}
\caption{\label{tab:DataSets}Datasets Details of Dublin Bus Delay Context.}
}
\vspace{-0.15cm}
\end{table}
\vspace{-0.1cm}

%
\noindent $\bullet$ \textbf{Validation:} 
%
Accuracy is measured by comparing predictions with real-time situations in cities, where results can be easily extracted and compared from all different approaches.

\vspace{0.1cm}
\noindent $\bullet$ \textbf{Semantic Impact:} 
%
Table~\ref{res:table1} reports the positive impact of using semantic embeddings (cf. columns with \checkmark) on all forecasting tasks, with an average improvement of $26.6\%$.
The embeddings naturally identify semantically (dis-)similar contexts by capturing temporal (in-)consistency(ies).
%
Thus, they 
help in building discriminating models, even for long-term-ahead forecasting as shown for $\Delta = 18$-hours 
with a $33.1\%$ gain.
The difference of results between Beijing and Dublin confirms the importance of semantic 
expressivity 
i.e., $40$+ times more 
axioms with a $71.5\%$ gain of accuracy for Dublin.

%
\begin{table}[h]
\vspace{-0.25cm}
\scriptsize{
\centering.
\begin{tabular}[bt]{|@{ }c@{ }|l|c|c||c|c||c|c|c|}\hline
\multirow{2}{*}{City}&\multirow{2}{*}{${\bf\mathcal{ID}} : Features$} & \multicolumn{2}{|c||}{$\Delta$=6 hours} & \multicolumn{2}{|c||}{$\Delta$=12 hours} & \multicolumn{2}{|c|}{$\Delta$=18 hours}\\ \cline{3-8}
&&\ding{55}&\checkmark&\ding{55}&\checkmark&\ding{55}&\checkmark \\ \hline \hline 
\multirow{4}{*}{\begin{sideways}Beijing\end{sideways}}&${\bf\mathcal{B}_1}: B_1$  & .351 & .398  & .344 & .441  &.261 &.342\\ \cline{2-8}
&${\mathcal{B}_2}: B_1$ + $B_2$ & .398 & .449 & .350 & .453 & .279 &.371 \\ \cline{2-8}
&${\mathcal{B}_3}: B_1$ + $B_3$  & .421 & .508  & .373 & .459  &.282 &.379\\ \cline{2-8}
&${\mathcal{B}_4}: B_1$ + $B_2$ + $B_3$  & .501 & .611  & .389 & .478  &.286 &.393\\ \hline
\multicolumn{2}{|c|}{Average Improvement (\%)} & \multicolumn{2}{|c||}{17.206}  & \multicolumn{2}{|c||}{25.890}  & \multicolumn{2}{|c|}{33.954}\\ \hline \hline 
\multirow{4}{*}{\begin{sideways}Dublin\end{sideways}}&${\bf\mathcal{D}_1}: D_1$  & .455 & .514  & .387 & .441  &.321 &.387\\ \cline{2-8}
&${\mathcal{D}_2}: D_1$ + $D_2$  & .534 & .688  & .499 & .553  &.361 &.497 \\ \cline{2-8}
&${\mathcal{D}_3}: D_1$ + $D_3$  & .601 & .701& .513 & .645  &.371&.547 \\ \cline{2-8}
&${\mathcal{D}_4}: D_1$ + $D_2$ + $D_3$  & .659 & .921  & .533 & .834  &.601 &.745 \\ \hline
\multicolumn{2}{|c|}{Average Improvement (\%)} & \multicolumn{2}{|c||}{24.550}  &\multicolumn{2}{|c||}{26.744}  & \multicolumn{2}{|c|}{32.408}\\ \hline 
\end{tabular}
\vspace{-0.15cm}
\caption{
Forecasting Accuracy of 
{\tt [A1]} with (\checkmark) and without (\ding{55}) Semantic Embeddings in Beijing and Dublin Contexts. 
%
%
}
\label{res:table1}
}
\end{table}

\vspace{-0.3cm}
\noindent $\bullet$ \textbf{Feature Impact:} 
Table~\ref{res:table1} emphasises an extra accuracy gain 
when increasing the number of features
i.e., average gain of $68.5\%$ accuracy from $1$ to $3$ features.
%

%
\vspace{0.1cm}
\noindent $\bullet$ \textbf{Concept Drift} is characterised by $48\%$ and $51\%$ of stream updates in respectively BAQ and DBD. We focus on $4$ levels of concept drifts, ranging from a $.2$ to $.8$ significance $\forall \Delta\in\{6, 12, 18\}$.
Level $0$ does not capture any change.
Figure \ref{res:conceptDrift} reports the proportion of severity levels in concept drift for BAQ and DBD e.g., $7\%$ are level-$.4$ for BAQ while $19\%$ are level-$.8$ for DBD.
Although accuracy clearly declined by increasing the severity level e.g., from $96\%$ (level-$.2$) to $21\%$ (level-$.8$) in DBD, 
semantic embeddings has shown to significantly boost accuracy.
More interestingly the more severity the higher improvement i.e., (average) $36\%$ to $56\%$ on level-$.4$ to $.8$.
Thus integrating semantics is a way forward to build machine learning models which are robust to changes, potential  erroneous sensor data and concept drifts.

\vspace{0.1cm}
\noindent $\bullet$ \textbf{Model Consistency Impact:} 
Figures \ref{res:consistent} and \ref{res:inconsistent} report accuracy of forecasting tasks on a {\bf H}igh and {\bf L}ow {\bf C}oncept {\bf D}rift versions of the Dublin and Beijing problems, noted HCD and LCD. $85\%$ and $15\%$ of snapshots are impacted by concept drift respectively in HCD and LCD.

\begin{figure}[h]
\vspace{0.5cm}
\centering
\includegraphics[angle=0,scale=0.38,bb=40 0 620 330]{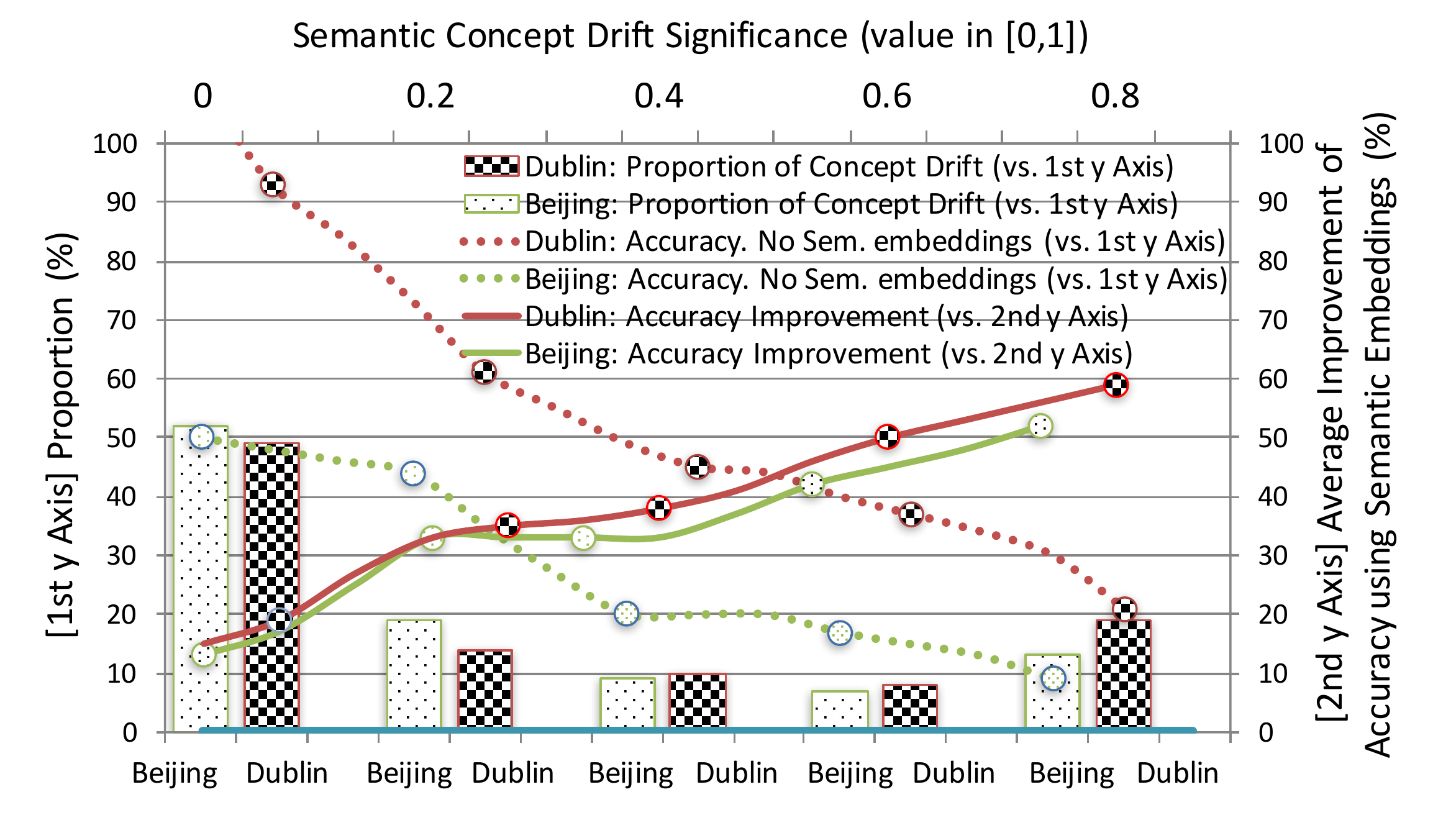}
\vspace{-0.4cm}
\caption{Forecasting Accuracy vs. 
Drift Significance.}
\label{res:conceptDrift}
\end{figure}
\vspace{-0.3cm}

{\tt [A1\Minus2]} is 
evaluated with 
$3$ variances of $(\varepsilon, \sigma_{\min}, \kappa)$: (i) consistent model with $(.9, .9, .1)$, (ii) mixed model with $(.5, .5, .5)$, (iii) inconsistent model with $(.1, .1, .9)$. ${\bf N}=1,500$.
%
Figure \ref{res:consistent} (resp. \ref{res:inconsistent}) reports that prediction with consistent (resp. inconsistent) samples outperforms models with inconsistent (resp. consistent) samples by about $318\%$ (resp. $456\%$) and $254\%$ (resp. $322\%$) in respectively Beijing and Dublin 
for LCD (resp. HCD). 
These results confirm the importance of semantic encoding, which support the encoding of concept drift and consistency properties in our approach.

%
\vspace{-0.1cm}
\begin{figure}[h]
\vspace{0.23cm}
\centering
\includegraphics[angle=0,scale=0.26,bb=362 0 620 340]{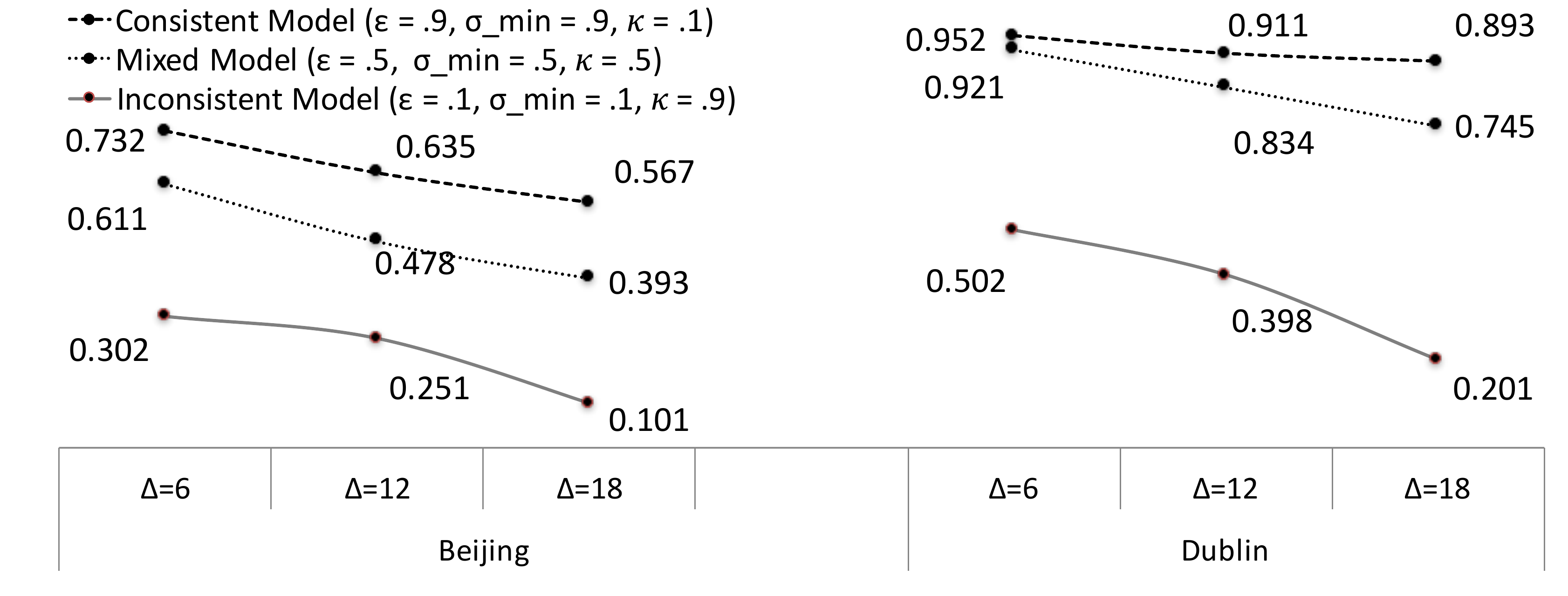}
\vspace{-0.4cm}
\caption{Model Consistency \& Forecasting Accuracy. Low Concept Drift. ($15\%$ of snapshots impacted by concept drift).}
\label{res:consistent}
\end{figure}

\vspace{-0.1cm}
\begin{figure}[h]
\vspace{-0.23cm}
\centering
\includegraphics[angle=0,scale=0.26,bb=362 0 620 340]{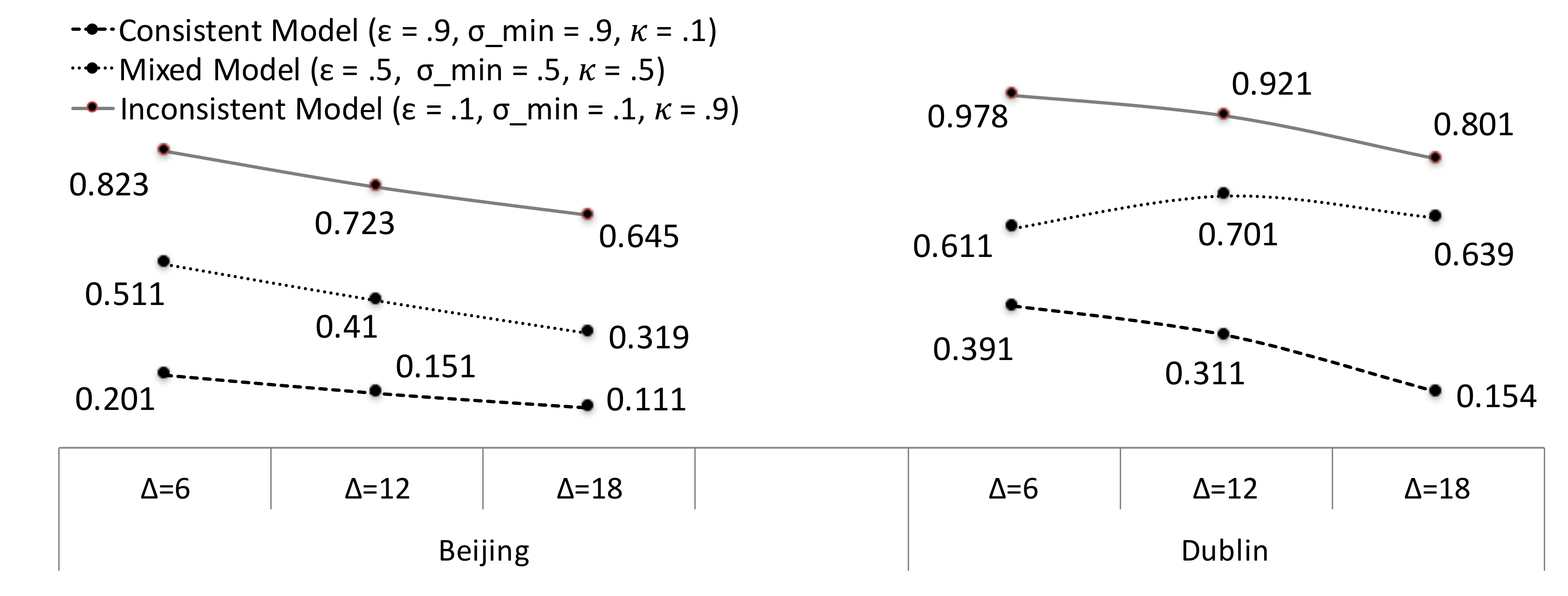}
\vspace{-0.4cm}
\caption{Model Consistency \& Forecasting Accuracy. High Concept Drift. ($85\%$ of snapshots impacted by concept drift).}
\label{res:inconsistent}
\end{figure}

\vspace{-0.2cm}
\noindent $\bullet$ \textbf{Baseline:} 
%
We compare 
our approach
${\bf{\mathcal{B}_i}}, {\bf{\mathcal{D}_{i,1 \leq i \leq 4}}}$ in Table \ref{res:table1} with (i) weighted {\bf S}tochastic {\bf G}radient {\bf D}escent (SGD), (ii) {\bf A}uto-{\bf R}egressive {\bf I}ntegrated {\bf M}oving {\bf A}verage (ARIMA), a standard time-series forecasting model \cite{saboia1977autoregressive}, and two methods addressing concept drift: (iii) {\bf A}daptive-{\bf S}ize {\bf H}oeffding {\bf T}ree (ASHT), (iv) {\bf AD}aptive {\bf WIN}dowing bagging (ADWIN) \cite{Bifet2009,Bifet2010}.
ARIMA considers one stream variable: BAQ index for Beijing and DBD for Dublin while SGD, ASHT and ADWIN 
use all features of ${\bf{\mathcal{B}_4}}, {\bf{\mathcal{D}_{4}}}$ and favour recent snapshots during learning.
The forecasted real value in $[0,5]$ is discretised back using our categories. 
Results with optimum parameters for {\tt [A1-2]} are reported. 
%
Figure \ref{res:baselines} emphasises that our approach (with $3$ levels of features: $\mathcal{B}_4$, $\mathcal{D}_4$) over-performs state-of-the-art methods. The more features the more accurate.
More interestingly classic learning algorithms do not generalise as well as {\tt [A1\Minus2]}
in presence of semantics although SGD, ASHT,
ADWIN integrate all features.
%
{\tt [A1\Minus2]} shows to be very robust with less variance. Experiments also demonstrate that semantic (in-)consistency matters more than recentness during learning.

%
\begin{figure}[h]
\vspace{-0.35cm}
\centering
\subfigure[Beijing Context]{
\label{res:baselines_dublin}
\includegraphics[width=0.23\textwidth]{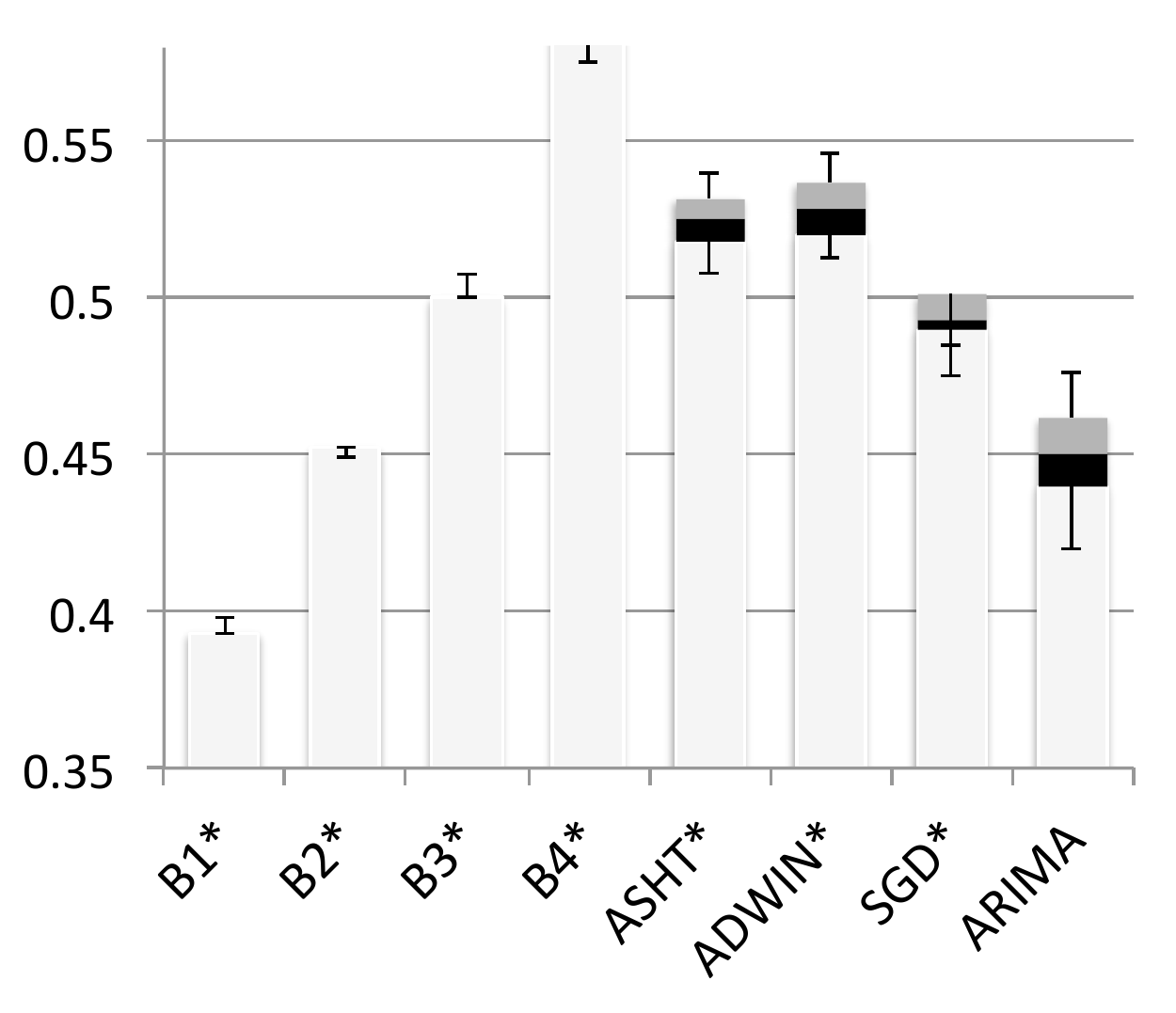}}
\subfigure[Dublin Context]{
\label{res:baselines_beijing}
\includegraphics[width=0.23\textwidth]{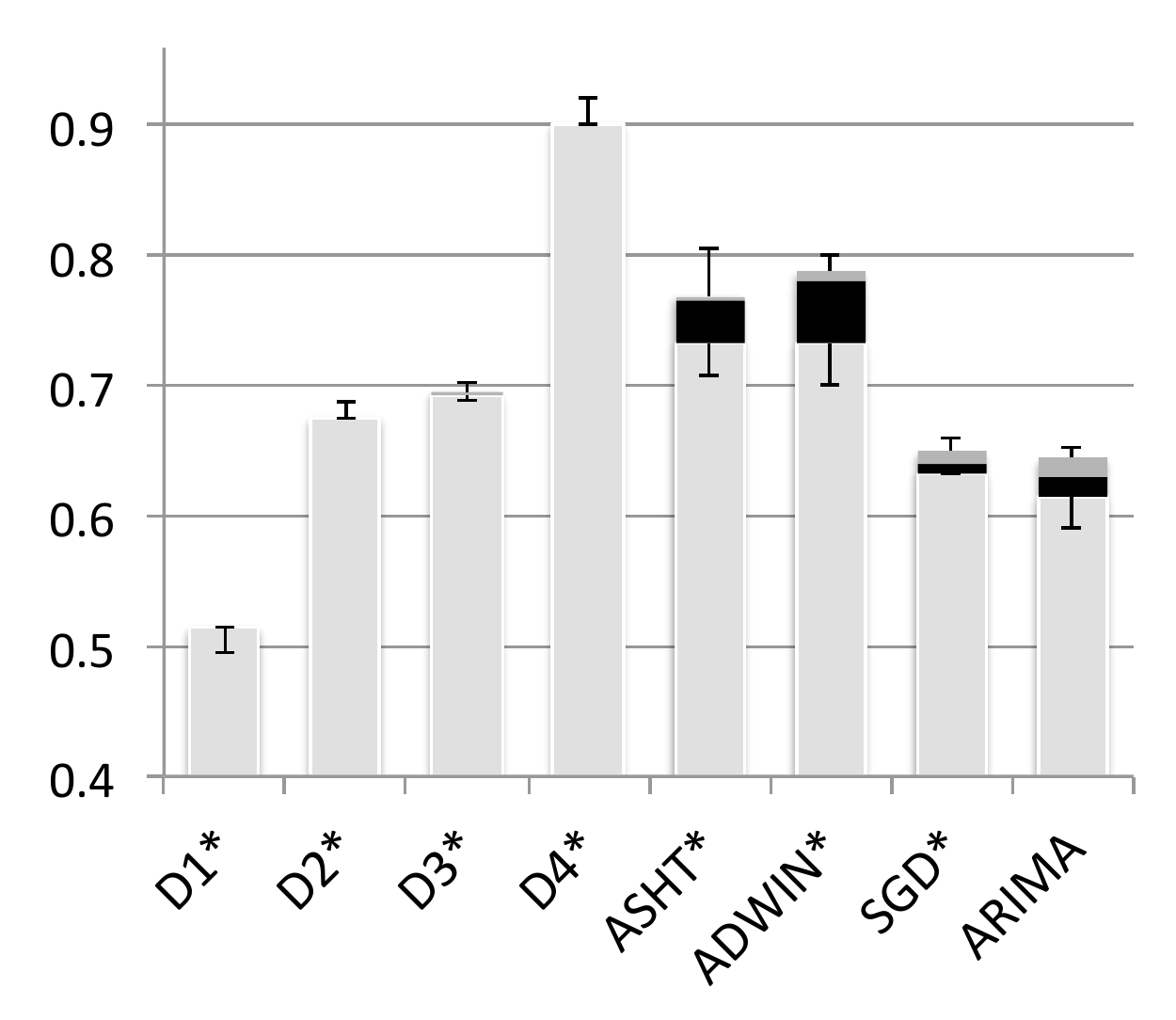}}
\vspace{-0.4cm}
\caption{Baseline Comparison of Forecasting Accuracy.
($\Delta = 6$, Approach$^{*}$: Approach with all features).
}
\label{res:baselines}
\vspace{-0.2cm}
\end{figure}

%
\noindent $\bullet$ \textbf{Lessons Learnt:} 
Adding semantics to classic learning model has clearly shown the positive impact on accuracy, specially in presence of concept drifts.
Our approach also demonstrates that the more semantic axioms the more robust is the model 
and hence the higher the accuracy.
%
Axiom numbers are critical as they drive and control the semantics of data in streams, which improve accuracy, concept drift detection but not scalability (not reported in the paper). It is worst with more expressive DLs due to consistency checks, and with limited impact on accuracy. 
%
Lightweight semantics such as RDF-S would highly limit the scope of our model given the omission of inconsistency checking cf. Figures \ref{res:consistent}-\ref{res:inconsistent}.

%
%
%
%
%
%
%
%

%
\vspace{-0.1cm}
\section{Conclusion}\label{sec:conclusion}

Our approach, exploiting the semantics of data streams, tackles the problem of learning and prediction with concept drifts. 
Semantic reasoning and machine learning have been combined by revisiting features embeddings 
as semantic embeddings i.e., vectors capturing consistency and entailment of any snapshot in ontology streams.
Such embeddings are then exploited in a context of supervised stream learning to learn models, which are robust to 
concept drifts i.e., sudden and abrupt (inconsistent) prediction changes. 
Our approach has been shown to be adaptable and flexible to basic learning algorithms. 
In addition to demonstrate accurate prediction with concept drifts in Dublin 
and Beijing 
forecasting applications, experiments have shown that encoding semantics in models is a way towards outperforming state-of-the-art approaches.

In future work we will investigate the impact of semantic embeddings in other Machine Learning models.

\bibliographystyle{named}
\bibliography{ijcai17-stream}

\end{document}